\documentclass{IEEEtran}

\usepackage[numbers,sort&compress]{natbib}
\usepackage{amsmath,amssymb,amsfonts,amsthm,mathtools}
\usepackage{bm}
\usepackage{xcolor}
\usepackage{hyperref}

\newtheorem{proposition}{Proposition}
\newtheorem{lemma}{Lemma}
\newtheorem{corollary}{Corollary}
\newtheorem{remark}{Remark}
\newtheorem{definition}{Definition}

\DeclareMathOperator{\tril}{tril}

\DeclareMathOperator{\diag}{diag}

\begin{document}

\title{Robust Automatic Differentiation of Square-Root Kalman Filters via Gramian Differentials}

\author{Adrien~Corenflos%
\thanks{The author is with the Department of Statistics, University of Warwick, Coventry, 
UK (e-mail: adrien.corenflos@warwick.ac.uk).}}

\maketitle

\begin{abstract}
Square-root Kalman filters propagate state covariances in Cholesky-factor form for numerical stability, and are a natural target for gradient-based parameter learning in state-space models.
Their core operation, triangularization of a matrix $M \in \mathbb{R}^{n \times m}$, is computed via a QR decomposition in practice, but naively differentiating through it causes two problems: the semi-orthogonal factor is non-unique when $m > n$, yielding undefined gradients; and the standard Jacobian formula involves inverses, which diverges when $M$ is rank-deficient.
Both are resolved by the observation that all filter outputs relevant to learning depend on the input matrix only through the Gramian $MM^\top$, so the composite loss is smooth in $M$ even where the triangularization is not.
We derive a closed-form chain-rule directly from the differential of this Gramian identity, prove it exact for the Kalman log-marginal likelihood and filtered moments, and extend it to rank-deficient inputs via a two-component decomposition: a column-space term based on the Moore--Penrose pseudoinverse, and a null-space correction for perturbations outside the column space of $M$.
\end{abstract}


\section{Introduction}
\label{sec:introduction}

We consider the linear Gaussian state-space model (LGSSM)
\begin{subequations}
\label{eq:ssm}
\begin{align}
  \mathbf{x}_{t+1} &= A\,\mathbf{x}_t + \mathbf{u}_t, \quad \mathbf{u}_t \sim \mathcal{N}(0, U), \\
  \mathbf{y}_t     &= B\,\mathbf{x}_t + \mathbf{v}_t, \quad \mathbf{v}_t \sim \mathcal{N}(0, V),
\end{align}
\end{subequations}
with state $\mathbf{x}_t \in \mathbb{R}^{d_x}$, and observation $\mathbf{y}_t \in \mathbb{R}^{d_y}$, for some transition and observation matrices $A$, $B$, respectively, and positive semi-definite covariance matrices $U$, and $V$.
The Kalman filter~\cite{kalman_new_1960} computes the posterior mean $\hat{\mathbf{x}}_{t|t}$ and covariance $P_{t|t}$ recursively, and evaluates the log-marginal likelihood $\mathcal{L}(\theta) =\log p(\mathbf{y}_{1:t} \mid \theta)$ for the parameter set $\theta = \{A, B, U, V, \hat{\mathbf{x}}_{0|0}, P_{0|0}\}$~\cite[see, e.g.,][Chapter 6]{sarkka_bayesian_2023}.
Gradient-based estimation of $\theta$ requires $\nabla_\theta\mathcal{L}$, which modern automatic differentiation~\cite[see, e.g.,][for a review]{baydin_automatic_2018} frameworks~\cite[e.g.,][]{bradbury_jax_2018} obtain by applying the chain rule through every primitive of the filter, in forward mode as a Jacobian-vector product~\citep[JVP,][Chapter 3]{griewank_evaluating_2008} or in reverse mode as a vector-Jacobian product~\citep[VJP,][Chapter 3]{griewank_evaluating_2008}.

Standard covariance updates are however numerically unstable over long horizons: rounding errors can cause $P_{t|t}$ to lose positive semi-definiteness.
Square-root filters~\cite{kaminski_discrete_1971,bierman_factorization_1977} avoid this by propagating a lower-triangular factor $S_{t|t}$ with $P_{t|t} = S_{t|t}S_{t|t}^\top$, and are increasingly used in parallel-in-time implementations~\cite{yaghoobi_parallel_2025}.
Every step of a square-root filter reduces to one call to the \emph{triangularization operator} $\mathcal{T}(M) = L$, where $M \in \mathbb{R}^{n \times m}$ and $L \in \mathbb{R}^{n \times n}$ satisfies $LL^\top = MM^\top$.
In practice $\mathcal{T}$ is computed via the thin QR decomposition of $M^\top$,
\begin{equation}
  M^\top = Q\,R, \quad Q \in \mathbb{R}^{m \times n},\; Q^\top Q = I_n,\; R \in \mathbb{R}^{n \times n},
  \label{eq:intro_qr}
\end{equation}
by setting $L = R^\top$.
When $M$ is square and full-rank, differentiating through the QR decomposition is well-defined~\cite[Section~3.1]{roberts_qr_2026} and has been used in this context, e.g., in~\cite{yaghoobi_parallel_2025}.
In general, however, two failure modes arise.
First, when $m > n$, the map $\mathcal{T}$ has no unique differential: $Q$ admits a continuous $(m-n)$-dimensional family of valid choices~\cite[Chapter 5]{golub_matrix_2013}, so there is no canonical JVP for $L$ or $M$; this occurs in the predict step whenever $\operatorname{rank}(U) > 0$ and in the update step whenever $d_y > 0$.
Second, the standard JVP formula for $R$ involves $R^{-1}$~\cite{seeger_auto-differentiating_2017}, which is singular whenever $M$ is rank-deficient; this arises when $U$ or $V$ is singular, or when $P_{t|t}$ degenerates during filtering.

Both failures are resolved by a single observation: the filter's marginal likelihood, filtered mean, and filtered covariance all depend on $L = \mathcal{T}(M)$ only through the Gramian $LL^\top = MM^\top$, so $\ell \circ \mathcal{T}$ factors through the smooth map $M \mapsto MM^\top$ even though $\mathcal{T}$ is not everywhere differentiable.
In view of this, we propose to differentiate $\mathcal{T}$ directly from the differential of the Gramian identity $LL^\top = MM^\top$, bypassing the QR decomposition entirely.
Our contributions are: (i) a closed-form JVP for $\mathcal{T}$ that is exact for any Gramian-dependent loss, recovers the classical QR-based formula wherever it exists, and extends to rank-deficient inputs via a column-space component based on the Moore--Penrose pseudoinverse together with a null-space correction term that accounts for perturbations outside $\operatorname{col}(L)$ (Section~\ref{sec:method}); (ii) a proof that all relevant Kalman filter outputs are Gramian-dependent, illustrated on a single predict-update step (Section~\ref{sec:single_step}); (iii) a VJP obtained for free by automatic transposition~\cite{radul_you_2023}, by virtue of the JVP being linear in the input tangent (Section~\ref{sec:vjp}); and (iv) an open-source implementation.

\section{Square-Root Kalman Filtering}
\label{sec:sqrtkf}

The square-root filter maintains the lower-triangular Cholesky factor $S_{t|t}$ satisfying $P_{t|t} = S_{t|t}S_{t|t}^\top$.
Each predict-update cycle is built from the following primitive.

\begin{definition}[Triangularization Operator]
For $M \in \mathbb{R}^{n \times m}$ with $r = \operatorname{rank}(M)$, the triangularization
$\mathcal{T}(M)$ is a matrix $L \in \mathbb{R}^{n \times n}$ satisfying $LL^\top = MM^\top$, of
the form
\begin{equation*}
  L = \begin{pmatrix} L_1 & 0 \\ L_2 & 0 \end{pmatrix},
\end{equation*}
where $L_1 \in \mathbb{R}^{r \times r}$ is lower-triangular with non-negative diagonal and
$L_2 \in \mathbb{R}^{(n-r) \times r}$ is arbitrary.
It is computed via the column-pivoted thin QR decomposition \eqref{eq:intro_qr} by setting
$L = R^\top$.
When $r = n$, $L_2$ is empty and $L = L_1$ is the usual lower-triangular Cholesky factor.
In all filter applications below, $M$ is a \emph{wide} matrix ($m \geq n$), consistently with
the block-augmented constructions in \eqref{eq:Mpred}--\eqref{eq:Mupd}.
\end{definition}

\noindent\textbf{Predict step.}
Form the augmented matrix
\begin{equation}
  M_{\mathrm{pred}}
    = \begin{bmatrix} A\,S_{t|t}^\top & (U^{1/2})^\top \end{bmatrix}
    \in \mathbb{R}^{d_x \times 2d_x},
  \label{eq:Mpred}
\end{equation}
where $U^{1/2} \in \mathbb{R}^{d_x \times d_x}$ satisfies $U = U^{1/2}(U^{1/2})^\top$.
Then $S_{t+1|t} = \mathcal{T}(M_{\mathrm{pred}})$, and one verifies $S_{t+1|t}S_{t+1|t}^\top = AP_{t|t}A^\top + U$.
When $U$ is rank-deficient, so is $M_{\mathrm{pred}}$, and computing gradients will require Definition~\ref{def:jvp} in Section~\ref{sec:method}.

\noindent\textbf{Update step.}
Write $V = V^{1/2}(V^{1/2})^\top$ with $V^{1/2} \in \mathbb{R}^{d_y \times d_y}$.
Form
\begin{equation}
  M_{\mathrm{upd}}
    = \begin{bmatrix} B\,S_{t+1|t} & V^{1/2} \\ S_{t+1|t} & 0 \end{bmatrix}
    \in \mathbb{R}^{(d_y+d_x) \times (d_y+d_x)}.
  \label{eq:Mupd}
\end{equation}
Partitioning the result $\mathcal{T}(M_{\mathrm{upd}})$ conformably as
\begin{equation}
  \mathcal{T}(M_{\mathrm{upd}}) =
    \begin{bmatrix} L_{11} & 0 \\ L_{21} & L_{22} \end{bmatrix},
  \quad L_{11} \in \mathbb{R}^{d_y \times d_y},\; L_{22} \in \mathbb{R}^{d_x \times d_x},
\end{equation}
one reads off the Cholesky factor of the predicted observation covariance $L_{11}$
(with $L_{11}L_{11}^\top = V + BP_{t+1|t}B^\top$) used to compute
$p(\mathbf{y}_t \mid \mathbf{y}_{1:t-1}; \theta)$, the Kalman gain
$K_t = L_{21}L_{11}^{-1}$,
and the posterior Cholesky factor $S_{t+1|t+1} = L_{22}$.
When $V$ is rank-deficient, so is $M_{\mathrm{upd}}$, and we will again need Definition~\ref{def:jvp} in Section~\ref{sec:method}.

\section{Differentiation via Gramian Differentials}
\label{sec:method}

\subsection{The Smooth-Factorization Argument}

\begin{lemma}[Smooth factorization]
\label{lem:smooth}
Let $\ell: \mathbb{R}^{n \times n}_{\mathrm{sym}} \to \mathbb{R}$ be smooth and suppose $\ell$ depends on $L = \mathcal{T}(M)$ only through the Gramian $\Sigma = LL^\top$.
Then $M \mapsto \ell(\mathcal{T}(M))$ factors as the composition of the polynomial map $M \mapsto MM^\top$ with a smooth function of $\Sigma$, and is therefore smooth in $M$ everywhere.
\end{lemma}

\begin{proof}
For any version of $\mathcal{T}$, $LL^\top = MM^\top$, so every branch yields the same value $\ell(\mathcal{T}(M)) = \tilde\ell(MM^\top)$ where $\tilde\ell(\Sigma) \coloneqq \ell(L)$ for any $L$ with $LL^\top = \Sigma$.
The map $M \mapsto MM^\top$ is polynomial, hence smooth.
\end{proof}

By Lemma~\ref{lem:smooth}, the JVP of $\ell \circ \mathcal{T}$ at $M$ in direction $dM$ equals the JVP of $M \mapsto \tilde\ell(MM^\top)$, which is well-defined everywhere.
Rather than differentiating through $MM^\top$ directly, we derive a surrogate tangent $dL$ satisfying the differential of $LL^\top = MM^\top$; by the chain rule through $\ell(L)$, any such $dL$ yields the correct JVP for $\ell$.

\subsection{Gramian Sufficiency}

\begin{proposition}[Gramian sufficiency]
\label{prop:gramian}
Let $\ell: \mathbb{R}^{n \times n}_{\mathrm{sym}} \to \mathbb{R}$ be a smooth function of the Gramian $\Sigma = LL^\top = MM^\top$, and let $dM$ be a tangent at $M$.
If $dL \in \mathbb{R}^{n \times n}$ satisfies the Gramian differential identity
\begin{equation}
  dL\,L^\top + L\,dL^\top = dM\,M^\top + M\,dM^\top,
  \label{eq:gramian_id}
\end{equation}
then the JVP of $\ell \circ \mathcal{T}$ at $M$ in direction $dM$ computed through $dL$ is exact.
\end{proposition}

\begin{proof}
We have $d\ell = \langle \nabla_\Sigma \ell,\, d\Sigma \rangle_F$.
Differentiating $\Sigma = LL^\top$ gives $d\Sigma = dL\,L^\top + L\,dL^\top$, and differentiating $\Sigma = MM^\top$ gives $d\Sigma = dM\,M^\top + M\,dM^\top$.
Equation~\eqref{eq:gramian_id} states these are equal, so both paths yield the same $d\ell$.
\end{proof}

The identity~\eqref{eq:gramian_id} is satisfied by any exact tangent to $\mathcal{T}$ at a full-rank point: it is simply the differential of the relation $LL^\top = MM^\top$.
Proposition~\ref{prop:gramian} shows that the entire class of tangents satisfying it yields exact JVPs, which allows us to construct a well-defined surrogate even where $\mathcal{T}$ is not differentiable.

\subsection{Deriving the Surrogate Tangent}

Let $P = LL^\dagger$ be the orthogonal projector onto $\operatorname{col}(L)$ and $P^\perp = I - P$ its complement.
Since $LL^\top = MM^\top$ the column spaces of $L$ and $M$ coincide, giving $PM = M$.
We seek $dL$ satisfying~\eqref{eq:gramian_id} by splitting the right-hand side as
\begin{equation}
\begin{split}
  dM\,M^\top + M\,dM^\top
    &= \underbrace{\bigl(P\,dM\,M^\top + M\,dM^\top P\bigr)}_{\text{column-space part}}\\
    &+ \underbrace{\bigl(P^\perp dM\,M^\top + M\,dM^\top P^\perp\bigr)}_{\text{null-space part}},
  \label{eq:rhs_split}
\end{split}
\end{equation}
and constructing a contribution for each part.

\paragraph{Column-space component}
From the QR decomposition $M = LQ^\top$, we have $M^\top = QL^\top$.
Substituting this into the first term of the column-space part yields $P\,dM\,M^\top = (LL^\dagger) dM (QL^\top) = L (L^\dagger dM Q) L^\top = L K L^\top$, where
\begin{equation}
  K = L^\dagger\,dM\,Q \in \mathbb{R}^{n \times n}.
  \label{eq:K}
\end{equation}
For the second term, we explicitly use the symmetry of the orthogonal projector $P = LL^\dagger$, such that $P = P^\top$. This yields $M\,dM^\top P = M\,dM^\top P^\top = (P\,dM\,M^\top)^\top = (L K L^\top)^\top = L K^\top L^\top$.
Consequently, the column-space part exactly equals $L(K + K^\top)L^\top$.
The equation $dN + dN^\top = K + K^\top$ for lower-triangular $dN$ has the unique solution
$dN = \tril(K + K^\top) - \diag(K)$,
since $\diag(K + K^\top) = 2\diag(K)$.
Setting
\begin{equation}
  dL_{\mathrm{col}} = L\,\bigl[\tril(K + K^\top) - \diag(K)\bigr],
  \label{eq:dLcol}
\end{equation}
one verifies
\begin{align}
    dL_{\mathrm{col}} L^\top + L\,dL_{\mathrm{col}}^\top
    &= L(K+K^\top)L^\top \nonumber\\
    &= P\,dM\,M^\top + M\,dM^\top P.
\end{align}

\paragraph{Null-space component}
Set $dL_{\mathrm{null}} = P^\perp dM\,Q$.
Using $QL^\top = M^\top$ and $LQ^\top = M$,
\begin{equation*}
  dL_{\mathrm{null}} L^\top + L\,dL_{\mathrm{null}}^\top
    = P^\perp dM\,M^\top + M\,dM^\top P^\perp,
\end{equation*}
which matches the null-space part of~\eqref{eq:rhs_split}.

\begin{definition}[Surrogate JVP]
\label{def:jvp}
The surrogate tangent for $\mathcal{T}$ at $M$ in direction $dM$ is
\begin{equation}
  dL = dL_{\mathrm{col}} + dL_{\mathrm{null}},
  \label{eq:dL}
\end{equation}
with $dL_{\mathrm{col}}$ as in \eqref{eq:dLcol}, $K$ as in \eqref{eq:K}, and $dL_{\mathrm{null}} = (I - LL^\dagger)\,dM\,Q$.
When $L$ is invertible, $P = I$ and $dL_{\mathrm{null}} = 0$, recovering the classical QR-based formula~\cite{seeger_auto-differentiating_2017}.
\end{definition}

\begin{proposition}[Verification]
\label{prop:verify}
The tangent $dL$ of Definition~\ref{def:jvp} satisfies the Gramian differential identity \eqref{eq:gramian_id}.
\end{proposition}

\begin{proof}
By construction (Section~\ref{sec:method}), $dL_{\mathrm{col}}$ satisfies the column-space part of~\eqref{eq:rhs_split} and $dL_{\mathrm{null}}$ satisfies the null-space part.
Summing:
\begin{align*}
  dL\,L^\top + L\,dL^\top
    &= \bigl(P\,dM\,M^\top + M\,dM^\top P\bigr)\\
    &+ \bigl(P^\perp dM\,M^\top + M\,dM^\top P^\perp\bigr) \\
    &= (P + P^\perp)\,dM\,M^\top + M\,dM^\top(P + P^\perp) \\
    &= dM\,M^\top + M\,dM^\top. \qedhere
\end{align*}
\end{proof}

\begin{remark}
When $\operatorname{rank}(M) = r < n$, the factor $L$ has $n-r$ zero columns and $L^\dagger$ is its Moore--Penrose pseudoinverse.
The column-space component $dL_{\mathrm{col}}$ handles perturbations $dM$ whose image under $Q$ lies within $\operatorname{col}(L)$, while the null-space component $dL_{\mathrm{null}}$ supplies the contribution from perturbations whose image lies outside $\operatorname{col}(L)$.
Together they constitute the unique solution to~\eqref{eq:gramian_id} of the form $dL_{\mathrm{col}} + dL_{\mathrm{null}}$, and $dL_{\mathrm{null}}$ vanishes when $L$ is full rank.
\end{remark}

\subsection{Application to a Single Predict-Update Step}
\label{sec:single_step}

We verify that Proposition~\ref{prop:gramian} applies to all relevant Kalman filter outputs.
The single-step log-likelihood contribution is
\begin{equation}
  \ell_t = -\tfrac{1}{2}\bigl[ d_y\ln(2\pi) + \ln\det(P_\nu) + \bm\nu_t^\top P_\nu^{-1}\bm\nu_t \bigr],
\end{equation}
where $\bm\nu_t = \mathbf{y}_t - B\hat{\mathbf{x}}_{t|t-1}$ is the innovation and $P_\nu = B P_{t|t-1} B^\top + V$ the innovation covariance.
The predicted mean $\hat{\mathbf{x}}_{t|t-1} = A\hat{\mathbf{x}}_{t-1|t-1}$ does not depend on $S_{t|t-1}$.
The innovation covariance $P_\nu$ is a linear function of $P_{t|t-1} = S_{t|t-1}S_{t|t-1}^\top$, the Kalman gain $K_t = P_{t|t-1}B^\top P_\nu^{-1}$ is a rational function of $P_{t|t-1}$, and the posterior covariance $P_{t|t} = P_{t|t-1} - K_t P_\nu K_t^\top$ depends on $P_{t|t-1}$ and $P_\nu$.
In every case the dependence on $S_{t|t-1}$ is solely through the Gramian $\Sigma_{t|t-1} = S_{t|t-1}S_{t|t-1}^\top = P_{t|t-1}$, never through the square-root factor $L$ directly.
By Proposition~\ref{prop:gramian}, any $dL$ satisfying \eqref{eq:gramian_id} therefore produces exact gradients of $\ell_t$, the filtered mean, and the filtered covariance with respect to $\theta$.

\section{Linearity and Reverse-Mode Differentiation}
\label{sec:vjp}

\begin{proposition}[Linearity]
\label{prop:linear}
The map $dM \mapsto dL$ of Definition~\ref{def:jvp} is linear in $dM$, with the primal quantities $L^\dagger$, $L$, and $Q$ held fixed.
\end{proposition}

\begin{proof}
Both $K = L^\dagger dMQ$ and $dL_{\mathrm{null}} = (I - LL^\dagger)dMQ$ are linear in $dM$.
The operations $\tril(\cdot)$, $\diag(\cdot)$, and left multiplication by $L$ are all linear.
Hence $dL = dL_{\mathrm{col}} + dL_{\mathrm{null}}$ is linear in $dM$.
\end{proof}

\begin{corollary}[VJP by transposition]
\label{cor:vjp}
Write $dL = \Phi(dM)$ for the linear map of Proposition~\ref{prop:linear}.
For an output cotangent $G_L = \partial\mathcal{L}/\partial L$, the VJP is $G_M = \Phi^\star(G_L)$, where $\Phi^\star$ is the adjoint defined by $\langle G_L, \Phi(dM)\rangle_F = \langle G_M, dM\rangle_F$.
Since $\Phi$ is a composition of primitive linear operations (matrix multiplications, $\tril$, $\diag$), AD frameworks derive $\Phi^\star$ mechanically by transposing each primitive~\cite{radul_you_2023}, without any additional derivation.
\end{corollary}

Proposition~\ref{prop:linear} is therefore the only additional property needed to obtain a full reverse-mode implementation from the forward-mode result; explicit derivation of the VJP is unnecessary.

\section{Empirical Validation}
\label{sec:experiments}
We consider a simple LGSSM $A = 0.9 I_4$, $U = 0.01 I_4$, $B = \left[I_2 \, \mathbf{0}\right] \in \mathbb{R}^{2 \times 4}$ and the family of covariances $V(\alpha) = \mathrm{diag}([1, \alpha^2])$, $\alpha \in [0,1]$, for $\mathbf{v}_t$, interpolating between full-rank and rank-one noise.
We set all observations $\mathbf{y}_t = 0$ so that they all are in the span of the model.
The tangents are plotted against the log-marginal likelihood in Figure~\ref{fig:gradient_stability} showcasing correctness despite the rank deficiency and mismatched dimensions.
The ``classical'' QR decomposition result is simply not shown as it returned invalid values throughout for the gradient.
\begin{figure}
    \centering
    \includegraphics[width=1\linewidth]{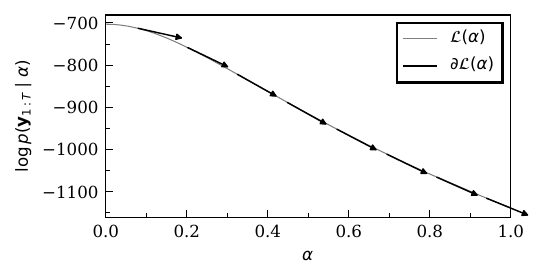}
    \caption{Log-marginal likelihood for the model in Section~\ref{sec:experiments} as a function of $\alpha$ as well as AD-computed tangents represented by arrows on the curve.
    Curve and tangents agree.}
    \label{fig:gradient_stability}
\end{figure}
The naive QR approach diverges, and therefore is not reported, while the approach presented here is correct and provides consistent gradients to the log-likelihood.

\section{Conclusion}
\label{sec:conclusion}

We have presented a robust JVP for the triangularization operator at the core of square-root Kalman filters.
The central insight is that the filter's marginal likelihood, filtered mean, and filtered covariance depend on the Cholesky factor only through the Gramian, so the composite map is smooth in the input matrix even where the triangularization itself is not.
Any tangent satisfying the Gramian differential identity \eqref{eq:gramian_id} propagates exact gradients, and we exhibit such a tangent in closed form in Definition~\ref{def:jvp}, verified by Proposition~\ref{prop:verify}.
Definition~\ref{def:jvp} decomposes the tangent into a column-space component, handled via the Moore--Penrose pseudoinverse, and a null-space correction that accounts for perturbations of $M$ outside $\operatorname{col}(L)$; together they satisfy~\eqref{eq:gramian_id} at any rank, recovering the classical formula when $L$ is invertible.
Proposition~\ref{prop:linear} establishes linearity of the resulting JVP in $dM$, from which the VJP follows by automatic transposition without additional derivation.
The method is immediately compatible with JAX, PyTorch, and similar AD frameworks for end-to-end gradient-based learning of state-space models and implemented in \cite{corenflos_cuthbert_2026} which uses JAX as its AD framework.

\section*{Acknowledgments}
The author would thank Dan Waxman for reporting the \href{https://github.com/state-space-models/cuthbert/issues/211}{bug} which prompted this short paper, Sahel Iqbal for carefully verifying the derivation and testing the resulting code, and Fatemeh Yaghoobi for checking the end manuscript.

\bibliographystyle{IEEEtranN}
\bibliography{main}

\end{document}